\documentclass[11pt,a4paper]{article}

\usepackage{amssymb,amsmath,amsfonts,amstext,amsthm}
\usepackage{bbm}
\usepackage{url}
\usepackage{fullpage}
\usepackage{times}
\usepackage{graphicx}
\usepackage{subcaption}

\usepackage{verbatim}

\newtheorem{theorem}{Theorem}
\newtheorem{lemma}[theorem]{Lemma}
\newtheorem{claim}[theorem]{Claim}

\theoremstyle{definition}

\newtheorem{defn}[theorem]{Definition}

\newcommand{\one}[1]{\mathbbm{1}\{#1\}}

\usepackage{enumitem}

\newcommand{\A}{\mathcal{A}}
\newcommand{\E}[1]{\mathbb{E}\left[ #1 \right]}

\newcommand{\X}{\mathbf{X}}
\DeclareMathOperator{\erf}{erf}
\newcommand{\ud}{\,\mathrm{d}}

\newcommand{\bigO}{\mathcal{O}}

\title{\bf Aggregating partial rankings with applications\\to peer grading in massive online open courses\thanks{This work was partially supported by the European Social Fund and Greek national funds through the research funding program Thales on ``Algorithmic Game Theory'', by the Caratheodory research grant E.114 from the University of Patras, and by the COST Action IC1205 on ``Computational Social Choice''.}}
\author{Ioannis Caragiannis\thanks{Computer Technology Institute ``Diophantus'' \& Department of Computer Engineering and Informatics, University of Patras, 26504 Rion, Greece. Email: {\tt caragian@ceid.upatras.gr}} 
\and 
George A. Krimpas\thanks{Department of Computer Engineering and Informatics, University of Patras, 26504 Rion, Greece. Email: {\tt krimpas@ceid.upatras.gr}} 
\and Alexandros A. Voudouris\thanks{Department of Computer Engineering and Informatics, University of Patras, 26504 Rion, Greece. Email: {\tt voudouris@ceid.upatras.gr}}
}
\date{}

\begin{document} 

\allowdisplaybreaks

\maketitle

\begin{abstract}
We investigate the potential of using ordinal peer grading for the evaluation of students in massive online open courses (MOOCs). 
According to such grading schemes, each student receives a few assignments (by other students) which she has to rank. Then, a global ranking (possibly translated into numerical scores) is produced by combining the individual ones.
This is a novel application area for social choice concepts and methods where the important problem to be solved is as follows: how should the assignments be distributed so that the collected individual rankings can be easily merged into a global one that is as close as possible to the ranking that represents the relative performance of the students in the assignment? Our main theoretical result suggests that using very simple ways to distribute the assignments so that each student has to rank only $k$ of them, a Borda-like aggregation method can recover a $1-\bigO(1/k)$ fraction of the true ranking when each student correctly ranks the assignments she receives. Experimental results strengthen our analysis further and also demonstrate that the same method is extremely robust even when students have imperfect capabilities as graders. We believe that our results provide strong evidence that ordinal peer grading can be a highly effective and scalable solution for evaluation in MOOCs.
\end{abstract}

\section{Introduction}\label{sec:intro}
Massive online open courses (MOOCs) such as Coursera and EdX have currently become a trend and have attracted significant funding from VCs and support from leading academics. Their vision is to use the Internet and provide (to huge numbers of students) an educational experience that is typical in courses targeted to small audiences in top-class Universities. Whether MOOCs will become the next big business over the Internet strongly depends on whether they will satisfy the fundamental need for easy and cheap access to high quality education without restrictions. An apparent bottleneck for their full deployment and success is the fact that assessment and grading with the classical means is extremely costly. A typical approach is to use closed type questions in exams or assignments so that grading can be done automatically. This is highly unsatisfactory when, as part of a course, one would like to evaluate the students' ability of proving a mathematical statement, or expressing their critical thinking over an issue, or even demonstrating their creative writing skills. Evaluating this ability is inherently a {\em human computation} task \cite{LvA11}.

An approach that has been proposed is to outsource the grading task to the students participating in the exam or assignent themselves; for example, they can be required to grade (a small number of) their peers' assignments as part of their own assignment \cite{PHC+13}. Of course, allowing the students to grade using cardinal scores is risky; they are not experienced in assessing their peers' performance in absolute terms\footnote{This is in contrast to the main assumption behind the reviewing systems that are used in academic conferences.} and they may have strong incentives to assign low scores to everybody in order to increase their own relative success in the assignment. An alternative that sounds feasible is to ask each student to provide a ranking of a small number of her peers' assignments and then compute a global ranking by merging the partial ones; this is known as {\em ordinal peer grading} (e.g., see \cite{RJ14,SBP+13}). Can this global ranking be in accordance to the objective comparison of students in terms of their performance in the assignment? Which are the necessary methods for this computation? And how accurate can this global ranking be? In this paper, we address these questions and provide both conceptual and technical answers. 

Merging individual rankings into a global one is the main goal of voting rules from {\em social choice theory}, where a set of voters provide full rankings over the available alternatives and a voting rule has to transform this input into a winning alternative or an aggregate ranking of the alternatives. At first glance, ordinal peer grading seems to be a natural application area for classical voting theory. Interestingly, its particular characteristics deviate from those usually studied in the voting literature. First, each voter is also an alternative. This is a rare assumption in social choice in works that focus mostly on incentives issues (e.g., see \cite{AFPT11, HM13}). Second, the input consists of partial rankings over small subsets of alternatives. The closest such approach in social choice is known as preference elicitation \cite{CS02} where simple queries are asked to each voter about their preferences; for example, in top-$k$ elicitation \cite{FO14}, each voter provides the partial ranking of the $k$ alternatives she likes the most. The (complexity) effects of using only partial rankings in voting have been studied under the possible and necessary winner problems (e.g., see \cite{XC11}). An important characteristic of ordinal peer grading is that the partial rankings have the same size and that each assignment is given to the same number of graders. And finally, there is an objective way to assess the ordinal peer grading outcome by comparing it to the objective comparison of the students in terms of performance in the assignment. This is close in spirit to recent approaches that use voting in order to {\em learn} a ground truth \cite{CPS13,CK12}, such as a winning alternative or an underlying true ranking. In our work, we deviate from these studies as well since we aim to learn the ground truth only approximately. So, ordinal peer grading is a setting where ideas and analysis techniques from human computation, voting, and learning are blended together in novel ways.

In particular, our model uses a grading scheme that asks each student to rank the assignments of $k$ other students. For fairness reasons, we restrict ourselves to grading schemes that distribute each assignment to exactly $k$ students. Unlike recent studies \cite{RJ14,SBP+13}, we investigate the potential of applying ordinal peer grading exclusively, i.e., without involving any professionals in grading. We assume that there is an underlying true (strict) ranking of the assignments (the ground truth) and we would like to recover correctly an as high as possible fraction of it using input from the students. We have two scenaria that determine the input. In the first one, we assume that, after the students have submitted their assignments, the instructor announces indicative solutions and grading instructions. Here, we make the simplifying assumption that each student grades the assignments in her bundle consistently to the ground truth (perfect grading). In a second scenario that is also assumed in \cite{RJ14}, we assume that grading is performed without any guidance by the instructor. Here, the natural assumption is that the quality of a student determines both her performance in the assignment and her grading ability. We have mostly focused on simple rank aggregation rules such as the adaptation of the classical Borda count \cite{Borda}, where the partial ranking provided by each grader is interpreted as follows: $k$ points are given to the assignment ranked first, $k-1$ points to the one ranked second, and so on. The global ranking is then computed by ordering the assignments in decreasing order of these Borda scores. We have also considered more aggregation rules which are described in detail in Sections \ref{sec:prelim} and \ref{sec:exp}.

Our technical contributions can be summarized as follows. In Section \ref{sec:borda}, we present a theoretical analysis of Borda when the partial rankings on input are consistent to the ground truth. We prove that using any way to distribute $k$ assignments per student, Borda recovers correctly an expected fraction of $1-\bigO(1/\sqrt{k})$ of the pairwise relations in the ground truth. If the distribution of the assignments has some particularly desired simple structure, an even better guarantee of $1-\bigO(1/k)$ is obtained. The independence of these results from the number of students is rather surprising. Our proofs exploit the beautiful theory of {\em martingales} in order to cope with dependencies between random variables that are involved in the analysis. To the best of our knowledge, this is the first application of martingales in social choice. We also present extensive experiments with Borda and other aggregation rules (in Section \ref{sec:exp}). Our findings further justify the robustness of Borda, even in the scenario of imperfect grading. For example, Borda is shown to recover more than $88\%$ of the ground truth by distributing $8$ assignments per student (with students having highly varying grading capabilities). Here, we borrow ideas from recent studies on voting and learning (e.g., \cite{CPS13}) and use noise models for the generation of random partial rankings whose distance from the ground truth depends probabilistically on the quality of the graders. En route, we provide some intuition about the problem (in Section \ref{sec:prelim}). We conclude with a discussion of (the many) possible extensions of our work in Section \ref{sec:discussion}.

\section{Problem statement, terminology and notation}\label{sec:prelim}
Let $\A$ denote a universe of $n$ {\em elements}. A collection ${\cal B}$ of subsets of $\A$ is called a {\em grading scheme} with parameters $n$ and $k\leq n$ (or $(n,k)$-grading scheme) if ${\cal B}$ consists of $n$ subsets of $\A$ called {\em bundles}, each bundle has size $k$, and each element of $\A$ belongs to exactly $k$ subsets of ${\cal B}$. To see the relation to peer grading, we can view the elements of the universe $\A$ as the $n$ papers of students participating in an assignment. Each bundle contains $k$ papers that will be graded by a distinct student. Of course, we require that no student will grade her own paper. This can be easily achieved by a matching computation.\footnote{Indeed, for every student $i$, there are $n-k$ bundles that do not contain her paper. Then, the bipartite graph that represents the information about the bundles that a student is allowed to grade is regular and, by Hall's matching theorem, has a perfect matching. This matching can be used to assign bundles of papers to students.}

Alternatively, we can represent the $(n,k)$-grading scheme with a bipartite graph $G=(U,V,E)$ which we will call $(n,k)$-{\em bundle graph}. The set of nodes $U$ has size $n$ and contains a distinct node for each element of $\A$. The set of nodes $V$ has size $n$ too and contains a node for each bundle of ${\cal B}$. The set of edges $E$ contains an edge $(u,v)$ connecting node $u\in U$ with node $v\in V$ if and only if the element corresponding to node $u$ belongs to the bundle corresponding to node $v$. Clearly, an $(n,k)$-bundle graph is $k$-regular. Actually, every $k$-regular bipartite graph has the same number $n$ of nodes in both bipartition sides and be used as an $(n,k)$-bundle graph. 

A {\em partial ranking} $\succ_b$ associated with a bundle $b\in {\cal B}$ is simply a ranking of the elements $b$ contains. We remark that $\succ_b$ is undefined for elements not belonging to ${\cal B}$. A {\em profile} is simply the collection that contains the partial ranking $\succ_b$ for each bundle $b$ of ${\cal B}$. An {\em aggregation rule} takes as input a profile of partial rankings and computes a complete ranking of all elements. A typical example is the following rule that extends Borda count from classical voting theory. Each element gets a score from each appearance in a partial ranking. The {\em Borda score} of an element is then the sum of the scores from all partial rankings. Within each partial ranking, a score of $k$ is given to the element that is ranked first, a score of $k-1$ to the element that is ranked second, and so on. The final complete ranking is computed by sorting the elements in decreasing order in terms of their Borda scores. We will use the term {\em Borda} to refer to this aggregation rule. Even though one can think of several different ways to resolve ties, we simply ignore ties in our theoretical analysis (Section~\ref{sec:borda}) and use uniformly random tie-breaking in our experiments (Section~\ref{sec:exp}). 

We have also considered another aggregation rule which we call {\em Random Serial Dictatorship} (RSD). The term is inspired by the well-known mechanism for house allocation markets \cite{AS98}. A complete ranking is computed gradually starting from an initially empty one. In a first {\em serial phase}, the partial rankings are considered in a random order. When considering a partial ranking, we copy to the global one all the pairwise relations that do not contradict (i.e., do not form cycles with) relations copied earlier. When all partial rankings have been considered, the global partial ranking is augmented by the pairwise relations implied due to transitivity (e.g., the pairwise relations $x\succ y$ and $y\succ z$ copied from two partial rankings imply that $x\succ z$ as well). Then, we use a second {\em random completion phase} to complete the global ranking as follows. In each step, we pick a random pair of elements whose relation has not been decided so far. We make this decision randomly and update all pairwise relations that this decision and the existing ones imply due to transitivity. We continue this way until all pairwise relations have been decided.

We are now ready to give the statement of the problem that we consider more formally. In general, we would like to use the grading schemes and aggregation rules in order to {\em learn} an unknown {\em ground truth}, i.e., a ranking of the elements representing their relative quality. A first question is whether the ground truth can be learnt {\em with certainty} when the partial rankings are consistent to it. In other words, we ask for an {\em order-revealing} grading scheme (and a corresponding order-revealing bundle graph) which defines the bundles in such a way that the partial rankings contain enough information so that all pairwise relations in the ground truth can be recovered with certainty. Unfortunately, order-revealing grading schemes have severe limitations. In particular, they should have the following too demanding property: for every pair of elements, there should be some bundle that contains both of them.\footnote{This property essentially asks for a $k$-regular bipartite graph of diameter at most $3$. Our order-revealing bundle graphs are known as Moore bipartite graphs, i.e., they are the smallest bipartite graphs of degree at least $k$ and of diameter at most $3$; see \cite{MS13} for a detailed survey on the degree-diameter problem.} Indeed, let ${\cal B}$ be an order-revealing grading scheme over a universe $\A$ of $n$ elements and assume that there are two elements $x$ and $y$ so that no bundle contains both $x$ and $y$. Now, consider a ranking $\succ$ that has $x$ and $y$ in the first two positions and let $\succ'$ be the ranking that differs from $\succ$ only in the order of $x$ and $y$. Clearly, the partial rankings within the bundles are identical in both cases and, as a result, there is no way to identify whether the ground truth is the ranking $\succ$ or the ranking $\succ'$. Notice that the above property implies that RSD combined with order-revealing grading schemes recovers the ground truth with certainty (and does not have to run the random completion phase). This is not the case for Borda unless any two elements co-exist in the same number of bundles (like in the bundle graphs constructed below).

Clearly, the maximum number of elements that belong to a bundle with $x$ is $k(k-1)$ and this number should be at least $n-1$ if we want $x$ to belong to some bundle with every other element. This immediately implies that order-revealing grading schemes should have bundles of size $\Omega(\sqrt{n})$. In sharp contrast to this disappointing observation, we will see that the goal of {\em approximate} order-revealing grading schemes is a very feasible one and leads to effective and scalable grading solutions in theory and practice. Interestingly, many of our findings make use of bundle graphs that are order-revealing; this is why we have included the following explicit construction of order-revealing grading schemes for particular values of the parameters $n$ and $k$ here.

Let $p\geq 1$ be a prime and let $\A$ be a universe with $n=p^2+p+1$ elements. We will construct the grading scheme ${\cal B}$ in which each bundle has size exactly $k=p+1$. Observe that these values for $n$ and $k$ satisfy the lower-bound condition mentioned above with equality. Rename the elements of $\A$ as $\A = \{u\} \cup \{v_i|i=0,...,p-1\} \cup \{w_{i,j}|i=0,...,p-1, j=0,...,p-1\}$ and define the bundles of ${\cal B}$ as follows:
\begin{itemize}
\item $F=\{u, v_0, v_1, ..., v_{p-1}\}$;
\item For $i=0,...,p-1$, $R_i=\{u\} \cup \{w_{i,j}| j=0,...,p-1\}$;
\item For $i=0,...,p-1$ and $s=0,...,p-1$, $C_{i,s}=\{v_s\} \cup \{w_{j,(i+j \cdot s)\bmod{p}}|j=0,...,p-1\}$.
\end{itemize}

An order-revealing $(7,3)$-bundle graph is depicted in Figure \ref{fig:example1}; it represents the following grading scheme ${\cal B}$. The underlying universe is $\A=\{1, 2, 3, 4, 5, 6, 7\}$ and ${\cal B}$ has the following seven $3$-sized bundles: $\{1, 2, 3\}$, $\{1, 4, 5\}$, $\{1, 6, 7\}$, $\{2, 4, 6\}$, $\{2, 5, 7\}$, $\{3, 4, 7\}$, and $\{3, 5, 6\}$. The numbering of nodes in set $V$ indicates an assignment of bundles to students for grading and, hence, nodes with the same number are not adjacent.

\begin{figure}[ht]
\centering
\includegraphics[scale=0.5]{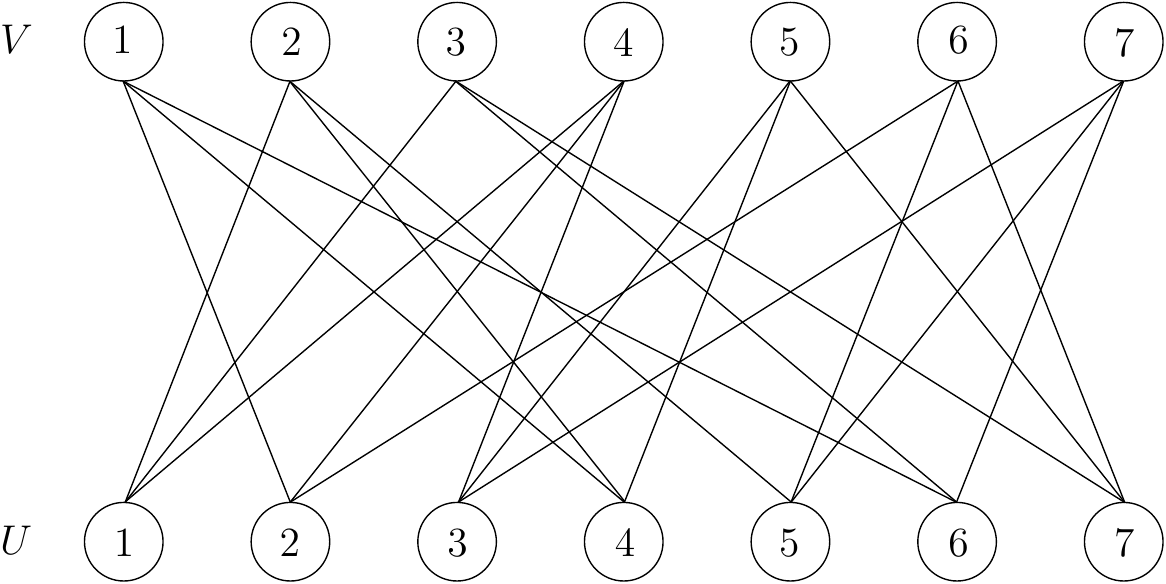}
\caption{An order-revealing $(7,3)$-bundle graph.}
\label{fig:example1}
\end{figure}

We prove the correctness of our construction using basic facts from number theory. 
\begin{lemma}\label{lem:ORGS-upper}
The above construction yields an order-revealing grading scheme.
\end{lemma}

\begin{proof}
Clearly, the above grading scheme consists of $n=p^2+p+1$ bundles of size $k=p+1$.
Also, observe that each element belongs to exactly $p+1$ bundles. Indeed, element $u$ belongs to sets $F$ and $R_i$ for $i=0, 1, ..., p-1$. Element $v_s$ belongs to sets $F$ and $C_{i,s}$ for $i=0, 1, ..., p-1$. Element $w_{i,j}$ belongs to sets $R_i$ and $C_{t,s}$ such that $j = (t+i \cdot s)\bmod{p}$.
  
We complete the proof by showing that for every pair $x,y \in \A$, there exists a bundle that contains both $x$ and $y$.
This is clearly true if one of $x$ and $y$ is $u$ or if both $x$ and $y$ belong to $F$ or to some $R_i$, for $i=0,...,p-1$. So, there are two more cases to be considered. 
First, assume that $x=v_s$ for $s \in \{0,...,p-1\}$ and $y=w_{j,\ell}$. 
Then, there exists an $i \in \{0,...,p-1\}$ such that $i+j \cdot s = \ell \pmod{p}$ and, hence, both $x$ and $y$ belong to set $C_{i,s}$. It remains to consider the case where $x=w_{i_1,j_1}$ and $y=w_{i_2,j_2}$ with $0 \leq i_1 < i_2 \leq p-1$. 
Then, there exists a unique $s \in \{0,...,p-1\}$ such that $(i_2-i_1)\cdot s=(j_2-j_1) \bmod{p}$. This follows from the facts that $p$ is prime and that any linear equation of the form $a\cdot z = b \pmod{n}$ has $\gcd(a,n)$ solutions if and only if $\gcd(a,n)$ divides $b$.
Now, set $i=(j_1-i_1\cdot s) \bmod{p}$ and observe that $i_1=(i+j_1 \cdot s)\bmod{p}$ and $i_2=(i+j_2 \cdot s)\bmod{p}$. Hence, both $x$ and $y$ belong to $C_{i,s}$ and the proof is complete.
\end{proof}

We now relax our requirements and seek for an approximate order-revealing grading scheme. Our aim is to use a bundle graph of simple structure and of very low (i.e., independent of $n$) degree and still be able to correctly recover a high fraction of the $n\choose 2$ pairwise relations in the ground truth. Our grading schemes will be randomized in the sense that we will always randomly permute the elements before associating them to nodes of set $U$ of the bundle graph; let $\pi: U\rightarrow {\cal A}$ denote this bijection (or permutation). Sometimes, in our experiments, the bundle graphs we use are themselves random. Much of our work (i.e., our theoretical analysis in Section \ref{sec:borda} as well as the first among the two sets of experiments reported in Section \ref{sec:exp}) has focused on the scenario where the partial rankings are consistent to the ground truth. Our second set of experiments in Section \ref{sec:exp} uses partial rankings that deviate from the ground truth according to a noise model.

\section{Analysis of Borda}\label{sec:borda}
In this section, we present our theoretical results. We assume that the $(n,k)$-bundle graph $G=(U,V,E)$ has $k\geq 3$ and $n\geq 3k(k-1)+2$. These are technical assumptions that do not affect the applicability of our results; recall that, in practice, we would like $n$ and $k$ to be huge and very small, respectively. Surprisingly, Borda correctly recovers a very large fraction of the ground truth as the next statement suggests. 

\begin{theorem}\label{thm:main}
When Borda is applied on partial rankings that are consistent to the ground truth, the expected fraction of correctly recovered pairwise relations is at least $1-\bigO\left(1/k\right)$ when the $(n,k)$-bundle graph has girth at least $6$, and at least $1-\bigO\left(1/\sqrt{k}\right)$ in general. 
\end{theorem}

We prove this theorem by relating the performance of Borda only to the degree $k$ and on a quantity $\eta(G)$ that characterizes the structure of the bundle graph. For the definition of $\eta(G)$, we need some notation; this will be heavily used throughout this section. Given two nodes $u,v$ of $U$, we use $\lambda_{u,v}$ to denote their common neighbourhood in $V$, i.e., $\lambda_{u,v}=|N(u)\cap N(v)|$. Observe that $\sum_{v\in U\setminus\{u\}}{\lambda_{u,v}}=k(k-1)$ since $G$ is $k$-regular. Also, we define the quantity $\theta_{u,v}$ as $\theta_{u,v} = 4\sum_{z \in N(N(u,v))\setminus\{u,v\}}{\left(\lambda_{u,z}+\lambda_{v,z}\right)^2}$. Then, 
$$\eta(G) = \frac{1}{n(n-1)}\sum_{u,v\in U}{\sqrt{\theta_{u,v}}},$$
where the sum runs over all ordered pairs of $u,v$ in $U$. 

Intuitively, the quantity $\eta(G)$ is small when, on average, the common neighbourhood between pairs of nodes is small. The extreme case is when the common neighbourhood consists of a single node; in this case, the graph has girth\footnote{The girth of a graph is the length of its smallest cycle.} at least $6$. The next lemma provides upper bounds on $\eta(G)$ that will be useful later. 

\begin{lemma}\label{lem:bounds}
For every $k$-regular bipartite graph $G$, $\eta(G)\leq \sqrt{8k(k-1)(4k-3)}$. Every $k$-regular bipartite graph $G$ of girth at least $6$ has $\eta(G)\leq 4\sqrt{k(k-1)}$.
\end{lemma}

\begin{proof}
Consider two nodes $u,v\in U$ of an arbitrary $k$-regular bipartite graph. We will show that $\theta_{u,v}$ is at most $8k(k-1)(4k-3)$. Consider the sets of nodes $N(u)\cap N(v)$, $N(u)\setminus N(v)$, and $N(v)\setminus N(u)$, and the edges connecting these nodes to $N(N(u,v))\setminus\{u,v\}$. Each edge from a node of $N(u)\cap N(v)$ to a node $z\in N(N(u,v))\setminus\{u,v\}$ contributes $2$ to the quantity $\lambda_{u,z}+\lambda_{v,z}$, which can be up to $2k$. Hence, each edge from a node of $N(u)\cap N(v)$ to a node $z\in N(N(u,v))\setminus\{u,v\}$ contributes at most $(2k)^2-(2k-2)^2=8k-4$ to the quantity $(\lambda_{u,z}+\lambda_{v,z})^2$ and there are $|N(u)\cap N(v)|(k-2)$ such edges. Similarly, each edge from a node of $N(u)\setminus N(v)$ and $N(v)\setminus N(u)$ to a node $z\in N(N(u,v))\setminus\{u,v\}$ contributes $1$ to the quantity $\lambda_{u,z}+\lambda_{v,z}$, which can be up to $2k-1$. Hence, each edge from a node of $N(u)\setminus N(v)$ or $N(v)\setminus N(u)$ to a node $z\in N(N(u,v))\setminus\{u,v\}$ contributes at most $(2k-1)^2-(2k-2)^2=4k-3$ to the quantity $(\lambda_{u,z}+\lambda_{v,z})^2$ and there are $2(k-|N(u)\cap N(v)|)(k-1)$ such edges. So, $\theta_{u,v}$ is bounded by $4$ times the total contributions to quantities $(\lambda_{u,z}+\lambda_{v,z})^2$ by the edges between $N(u,v)$ and $N(N(u,v))\setminus\{u,v\}$, i.e., by $4(|N(u)\cap N(v)|(k-2)(8k-4) + 2(k-|N(u)\cap N(v)|)(k-1)(4k-3)) \leq 8k(k-1)(4k-3)$.

Now, assume that the graph has girth at least $6$; this means that $\lambda_{u,z}+\lambda_{v,z}\leq 2$ for any node $z\in N(N(u,v))\setminus \{u,v\}$, otherwise $z$ would be in a $4$-cycle with either $u$ or $v$. We will show that $\theta_{u,v}\leq 16k(k-1)$. Each node $z\in N(N(u,v))\setminus \{u,v\}$ can be adjacent to either one node of $N(u)\cap N(v)$ or (exclusive) to at most one node of $N(u)\setminus N(v)$ and at most one node of $N(v)\setminus N(u)$. Among the nodes in $N(N(u,v))\setminus \{u,v\}$, denote by $D_2$ the ones that are adjacent to one node from $N(u)\setminus N(v)$ and to one node from $N(v)\setminus N(u)$. So, any node $z$ that is among the $|N(u)\cap N(v)|(k-2)$ neighbours of $N(u)\cap N(v)$ in $N(N(u,v))\setminus \{u,v\}$ or belongs to $D_2$ has $\lambda_{u,z}+\lambda_{v,z}=2$. Any node $z$ among the remaining $2(k-|N(u)\cap N(v)|)(k-1)-2|D_2|$ nodes of $N(N(u,v))\setminus \{u,v\}$ has $\lambda_{u,z}+\lambda_{v,z}=1$. Hence, $\theta_{u,v}=4(4(|N(u)\cap N(v)|(k-2)+|D_2|)+2(k-|N(u)\cap N(v)|)(k-1)-2|D_2|)=8k(k-1)+8(|N(u)\cap N(v)|(k-2)+|D_2|)-8|N(u)\cap N(v)|$. The second part of the lemma follows by observing that the quantity $|N(u)\cap N(v)|(k-2)+|D_2|$ is the number of nodes $z$ of $N(N(u,v))\setminus\{u,v\}$ with $\lambda_{u,z}+\lambda_{v,z}=2$ which cannot be higher than $k(k-1)$.
\end{proof}

The important step in the proof of Theorem \ref{thm:main} is to focus on two elements $a_r$ and $a_q$ with ranks (positions) $r<q$ in the ground truth and to bound from above the probability that the difference in their Borda scores is inconsistent to their rank difference. This will require to take care of several subtle dependencies among the random variables involved. We will do so by exploiting the beautiful theory of martingales and a well-known tail inequality about them. The necessary background from martingale theory is presented below; the interested reader can refer to the textbooks \cite{MU05} and \cite{MR95} for an introduction to martingales and their applications.
\begin{defn}
A sequence of random variables $Z_0, Z_1, ..., Z_m$ is a {\em martingale} with respect to a second sequence of random variables $X_1, X_2, ..., X_m$ if for every $i=1,...,m$, it holds that $\E{Z_i|X_1,...,X_i}=Z_{i-1}$. 
\end{defn}
The next definition provides a general way to define martingales associated with {\em any} random variable and was first used by Doob \cite{Doob53}.
\begin{defn}
Consider a random variable $W$ and a sequence of random variables $X_1,...,X_m$. Then, the sequence of random variables $Z_0,...,Z_m$ such that $Z_0=\E{W}$ and $Z_i=\E{W|X_1,...,X_i}$ for every $i=1, ..., m$ is a martingale, called a {\em Doob martingale}.
\end{defn}
We can now present a powerful tail inequality for martingales that is known as Azuma-Hoeffding inequality (see Azuma \cite{A67} and Hoeffding \cite{H63}).
\begin{lemma}[Azuma-Hoeffding inequality]\label{lem:ah}
Let $Z_0, Z_1, ..., Z_m$ be a martingale with $|Z_i-Z_{i-1}|\leq c_i$ for $i=1, ..., m$.
Then, for all $t\geq 0$, it holds that
\begin{align*}
\Pr[ Z_m - Z_0 \leq -t] \leq \exp\left(-\frac{t^2}{2 \sum_{i=1}^m c^2_i}\right).
\end{align*}
\end{lemma}

We are now ready to show that the probability that the Borda score of a high-rank element is larger than the Borda score of a low-rank element is small. Importantly, it turns out that this probability decreases exponentially in terms of the rank difference. We will first study such phenomena under particular conditions on our bijection $\pi$.
\begin{lemma}\label{lem:Borda-distribution}
Let $u,v\in U$, and consider the two elements $a_r, a_q \in \A$ with ranks $r<q$ in the ground truth. Let $W_{r,q}$ be the random variable denoting the difference of the Borda score of $a_r$ minus the Borda score of $a_q$ and let $\Gamma_{u,v}^{r,q}$ be the event that $\pi(u)=a_r$ and $\pi(v)=a_q$. Then, 
$$\E{W_{r,q}|\Gamma_{u,v}^{r,q}} = \left(k(k-1)-\lambda_{u,v}\right)\frac{q-r-1}{n-2}+\lambda_{u,v}$$ 
and
$$\Pr[W_{r,q} \leq  0|\Gamma_{u,v}^{r,q}] \leq \exp\left(-\frac{\E{W_{r,q}|\Gamma_{u,v}^{r,q}}^2}{2\theta_{u,v}}\right).$$ 
\end{lemma}

\begin{proof}
We begin the proof by computing the expectation of the Borda scores. Element $a_r$ gets one point for each bundle it belongs to plus one additional point for each appearance of an element with rank higher than $r$ in the bundles $a_r$ belongs to. Assuming that $\pi(u)=a_r$ and $\pi(v)=a_q$, there are $\lambda_{u,v}$ appearances of $a_q$ in the bundles of $a_r$ and $k(k-1)-\lambda_{u,v}$ appearances of elements different than $a_r$ and $a_q$; each of them has probability $\frac{n-r-1}{n-2}$ to have higher rank than $r$. Hence, the expected Borda score of element $a_r$ is $k+\left(k(k-1)-\lambda_{u,v}\right)\frac{n-r-1}{n-2}+\lambda_{u,v}$. Similarly, element $a_q$ gets one point for each bundle it belongs to plus one additional point for each appearance of an element with rank higher than $q$. There are $k(k-1)-\lambda_{u,v}$ appearances of elements different than $a_r$ and $a_q$ in bundles of $a_q$ and each of them has rank higher than $q$ with probability $\frac{n-q}{n-2}$. Hence, the expected Borda score of element $a_q$ is $k+\left(k(k-1)-\lambda_{u,v}\right)\frac{n-q}{n-2}$, and the expectation of the difference $W_{r,q}$ is indeed
\begin{eqnarray*}
\E{W_{r,q}|\Gamma_{u,v}^{r,q}} &=& \left(k(k-1)-\lambda_{u,v}\right)\frac{q-r-1}{n-2}+\lambda_{u,v}. 
\end{eqnarray*}

Given $\Gamma_{u,v}^{r,q}$, define $S=N(N(u,v))\setminus\{u,v\}$ to be the set of nodes in $G$ that are at distance exactly $2$ from $u$ or $v$ (not including $u$ and $v$); notice that $|S| \leq 2k(k-1)$. Now, consider an arbitrary ordering $o:[|S|] \rightarrow S$ of the nodes of $S$ and let $X_i$ be the random variable denoting the rank of the element $\pi(o(i))$. Using the random variables $X_i$ and the random variable $W_{r,q}$, we define the Doob martingale $Z_0, Z_1, ..., Z_{|S|}$ such that $Z_0=\E{W_{r,q}|\Gamma_{u,v}^{r,q}}$ and $Z_i=\E{W_{r,q}|\Gamma_{u,v}^{r,q},X_1,...,X_i}$ (hence, given $\Gamma_{u,v}^{r,q}$, $W_{r,q} = Z_{|S|}$). The next technical lemma bounds the difference $|Z_i-Z_{i-1}|$ for $i=1, ..., |S|$. 

\begin{lemma}\label{lem:Bounded-differences}
For every $i=1,...,|S|$, it holds that $|Z_i-Z_{i-1}|\leq 2\left(\lambda_{u,o(i)}+\lambda_{v,o(i)}\right)$.
\end{lemma}

\begin{proof}
Throughout this proof, all random variables and probabilities are conditioned on the event $\Gamma_{u,v}^{r,q}$, even if, in order to simplify notation, we do not explicitly write so. 

For every node $w\in S$, denote by $\mu_{u,v,w}= |N(u) \cap N(v)\cap N(w)|$ the number of common neighbours between $u$, $v$, and $w$. We can now express $W_{r,q}$ using the following observations: the Borda score difference 
\begin{itemize}
\item increases for each appearance of element $a_q$ in the same bundle with $a_r$, 
\item for each appearance of element $\pi(o(j))$ in a bundle containing $a_r$ but not $a_q$ provided that the rank of $\pi(o(j))$ is higher than $r$, and 
\item for each appearance of an element 
$\pi(o(j))$ in a bundle containing both $a_r$ and $a_q$ provided that the rank of $\pi(o(j))$ is between $r$ and $q$, 
and 
\item decreases for each appearance of element $\pi(o(j))$ in a bundle containing $a_q$ but not $a_r$ provided that the rank of $\pi(o(j))$ is higher than $q$. 
\end{itemize}
Using our notation $\lambda_{u,v}$ and $\mu_{u,v,o(j)}$, we have
\begin{eqnarray*}
W_{r,q} &=& \lambda_{u,v}+\sum_{j=1}^{|S|}{\left(\lambda_{u,o(j)}-\mu_{u,v,o(j)}\right)\one{X_j>r}}
+\sum_{j=1}^{|S|}{\mu_{u,v,o(j)}\one{r<X_j<q}}\\
&& -\sum_{j=1}^{|S|}{\left(\lambda_{v,o(j)}-\mu_{u,v,o(j)}\right)\one{X_j>q}}\\
&=& \lambda_{u,v}+\sum_{j=1}^{|S|}{\left(\lambda_{u,o(j)}\one{X_j>r}-\lambda_{v,o(j)}\one{X_j>q}\right)}.
\end{eqnarray*}

Denoting by $\X_i$ the sequence $X_1, ..., X_i$, we have that the difference $|Z_i-Z_{i-1}|$ is
\begin{eqnarray}\label{eq:1}\nonumber
Z_i-Z_{i-1} &=& \sum_{j=i}^{|S|}\lambda_{u,o(j)}\left(\Pr[X_j>r|\X_i]-\Pr[X_j>r|\X_{i-1}]\right) \\
&& -\sum_{j=i}^{|S|}\lambda_{v,o(j)}\left(\Pr[X_j>q|\X_i]-\Pr[X_j>q|\X_{i-1}]\right).
\end{eqnarray}

Once the values of $X_1, ..., X_{i-1}$ are determined, let $x$ and $y$ be the number of available ranks from $[n]\setminus\{r,q,X_1,...,X_{i-1}\}$ that are between $r$ and $q$ and higher than $q$, respectively. Hence, for $j=i, ..., |S|$, we have
$$\Pr[X_j>r|\X_{i-1}]=\frac{x+y}{n-i-1},$$
$$\Pr[X_j>q|\X_{i-1}]=\frac{y}{n-i-1},$$
and for $j=i+1, ..., |S|$, we have
$$\Pr[X_j>r|\X_i] = \frac{x+y-\one{X_i>r}}{n-i-2},$$
$$\Pr[X_j>q|\X_i] = \frac{y-\one{X_i>q}}{n-i-2}.$$
Now, (\ref{eq:1}) yields
\begin{eqnarray*}
Z_i-Z_{i-1} &=& \lambda_{u,o(i)}\left(\one{X_i>r}-\frac{x+y}{n-i-1}\right)-
\lambda_{v,o(j)}\left(\one{X_i>q}-\frac{y}{n-i-1}\right)\\
&&+\sum_{j=i+1}^{|S|}{\lambda_{u,o(j)}\left(\frac{x+y-\one{X_i>r}}{n-i-2}-\frac{x+y}{n-i-1}\right)}\\
&&+\sum_{j=i+1}^{|S|}{\lambda_{v,o(j)}\left(\frac{y-\one{X_i>q}}{n-i-2}-\frac{y}{n-i-1}\right)}\\
&=& \left(\lambda_{u,o(i)}-\frac{\sum_{j=i+1}^{|S|}{\lambda_{u,o(j)}}}{n-i-2} \right)\left(\one{X_i>r} -\frac{x+y}{n-i-1} \right) \\
&& + \left(\lambda_{v,o(i)}-\frac{\sum_{j=i+1}^{|S|}{\lambda_{v,o(j)}}}{n-i-2} \right)\left(\frac{y}{n-i-1}-\one{X_i>q}\right)
\end{eqnarray*}
The second and fourth parenthesis in the above expression are obviously between $-1$ and $1$. Recall that $\sum_{j=i+1}^{|S|}{\lambda_{u,o(j)}} \leq k(k-1)$ and $\sum_{j=i+1}^{|S|}{\lambda_{v,o(j)}} \leq k(k-1)$. Also, by the definition of $S$, $\lambda_{u,o(i)}+\lambda_{v,o(i)}\geq 1$, for every $i=1,...,|S|$. Combined with our assumption that $n\geq 3k(k-1)+2$, these properties imply that the first parenthesis is between $-\max\{\lambda_{u,o(i)},1\}$ and $\max\{\lambda_{u,o(i)},1\}$, and the third one is between $-\max\{\lambda_{v,o(i)},1\}$ and $\max\{\lambda_{v,o(i)},1\}$. The lemma follows since $|\max\{\lambda_{u,o(i)},1\}+\max\{\lambda_{v,o(i)},1\}| \leq 2(\lambda_{u,o(i)}+\lambda_{v,o(i)})$.
\end{proof}

Lemma \ref{lem:Borda-distribution} then follows by applying the Azuma-Hoeffding inequality (Lemma \ref{lem:ah}) with $t=\E{W_{r,q}|\Gamma_{u,v}^{r,q}}$ and using Lemma \ref{lem:Bounded-differences} to bound the difference $|Z_i-Z_{i-1}|$.
\end{proof}

The proof of Theorem \ref{thm:main} can now be completed using Lemmas \ref{lem:bounds} and \ref{lem:Borda-distribution}.

\begin{proof}[Proof of Theorem \ref{thm:main}]
Consider the pair of elements with true ranks $r$ and $q$ so that $r<q$. The correct pairwise relation between the two elements will be recovered when the Borda score of the low-rank element is higher than the Borda score of the high-rank one (there is the additional case where the two elements are tied and the tie is resolve in favour of the low-rank element but we will ignore this case; this will only make our result stronger). Again, $W_{r,q}$ will be the random variable denoting the difference between the Borda scores of the low- and high-rank elements. Then, by Lemma \ref{lem:Borda-distribution} the probability that the relation between the elements with ranks $r$ and $q$ is correctly recovered is
\begin{eqnarray*}
\Pr[W_{r,q}>0] &=& 1-\sum_{u,v\in U}{\left(\Pr[W_{r,q}\leq 0|\Gamma_{u,v}^{r,q}]\Pr[\Gamma_{u,v}^{r,q}]\right)}\\
&\geq & 1- \frac{1}{n(n-1)}\sum_{u,v\in U}{\exp\left(-\frac{\E{W_{r,q}|\Gamma_{u,v}^{r,q}}^2}{2\theta_{u,v}}\right)}\\
&=& 1- \frac{1}{n(n-1)}\sum_{u,v\in U}{e^{-\left(\beta(u,v) y(q-r)+\delta(u,v)\right)^2}},
\end{eqnarray*}
where $\beta(u,v) = \frac{k(k-1)-\lambda_{u,v}}{\sqrt{2\theta_{u,v}}}$, $\delta(u,v)=\frac{\lambda_{u,v}}{\sqrt{2\theta_{u,v}}}$, and $y(t)=\frac{t-1}{n-2}$. Now, denoting the expected number of correctly recovered pairwise relations by $C$, we have
\begin{eqnarray*}
C &=& \sum_{r=1}^{n-1}{\sum_{q=r+1}^n{\Pr[W_{r,q}>0]}}\\
&\geq & \sum_{r=1}^{n-1}{\sum_{q=r+1}^n{\left(1- \frac{1}{n(n-1)}\sum_{u,v\in U}{e^{-\left(\beta(u,v) y(q-r)+\delta(u,v)\right)^2}}\right)}}\\
&=& \frac{n(n-1)}{2} - \frac{1}{n(n-1)}\sum_{u,v\in U}{\sum_{d=1}^{n-1}{(n-d)e^{-\left(\beta(u,v) y(d)+\delta(u,v)\right)^2}}}\\
&\geq & \frac{n(n-1)}{2} - \sum_{u,v\in U}{\int_0^1{(1-y)e^{-(\beta(u,v) y + \delta(u,v))^2} \ud{y}}}.
\end{eqnarray*}
We will estimate the (Gaussian) integral using the following claim.
\begin{claim}\label{claim:gaussian}
Let $\beta>0$ and $\delta\geq 0$. Then, $\int_0^1{(1-y)e^{-(\beta y+\delta)^2}\ud{y}} \leq \frac{\beta+\delta}{2\beta^2}\sqrt{\pi}$.
\end{claim}
\begin{proof}
Denote by $\erf(y)=\frac{2}{\sqrt{\pi}}\int_0^y{e^{-t^2}\ud{t}}$ the error function. Then, we can verify by tedious calculations that
\begin{eqnarray*}
\int_0^1{(1-y)e^{-(\beta y+\delta)^2}\ud{y}} &=& \frac{\beta+\delta}{2\beta^2}\sqrt{\pi}\left(\erf(\beta+\delta)-\erf(\beta)\right)+\frac{1}{2\beta^2}\left(e^{-(\beta+\delta)^2}-e^{-\beta^2}\right)\\
&\leq & \frac{\beta+\delta}{2\beta^2}\sqrt{\pi},
\end{eqnarray*}
where the inequality follows since the error function $\erf(y)$ takes values in $[0,1]$ when $y\geq 0$.
\end{proof}
Now, we use Claim \ref{claim:gaussian} and the facts $\beta(u,v)=\frac{k(k-1)-\lambda_{u,v}}{\sqrt{2\theta_{u,v}}}\leq \frac{k(k-2)}{\sqrt{2\theta_{u,v}}}$ and $\delta(u,v)=\frac{\lambda_{u,v}}{\sqrt{2\theta_{u,v}}}$ to obtain
\begin{eqnarray*}
C &\geq & \frac{n(n-1)}{2}-\sum_{u,v\in U}{\frac{\beta(u,v)+\delta(u,v)}{2\beta(u,v)^2} \sqrt{\pi}}\\
&\geq & \frac{n(n-1)}{2}-\frac{k-1}{k(k-2)^2}\sqrt{\frac{\pi}{2}}\sum_{u,v\in U}{\sqrt{\theta_{u,v}}}\\
&=& \frac{n(n-1)}{2}\left(1-\frac{k-1}{k(k-2)^2}\sqrt{2\pi}\eta(G)\right).
\end{eqnarray*}
Now, the theorem follows by Lemma \ref{lem:bounds}. Recall that $\eta(G)$ is at most $\sqrt{8k(k-1)(4k-3)}$ for every $k$-regular bipartite graph $G$ and at most $4\sqrt{k(k-1)}$ when $G$ has girth at least $6$. Using the assumption that $k\geq 3$, we obtain that the rightmost parenthesis in the above expression becomes at least $1-\frac{48\sqrt{2\pi}}{\sqrt{k}}$ and $1-\frac{16\sqrt{3\pi}}{k}$, respectively.
\end{proof}

\section{Experimental evaluation}\label{sec:exp}
We now describe two sets of experiments that we have conducted.\footnote{All experiments presented in this section have been conducted in an Intel 12-core i7 machine with 32Gb of RAM running Windows 7. Our methods have been implemented in Matlab R2013a.} In the first one, we have studied perfect grading with Borda and RSD. We have considered three different types of bundle graphs. The first type is that of random $k$-regular bipartite graphs. We build these graphs by picking $k$ perfect matchings in the complete bipartite graph $K_{n,n}$ as follows. For each node of $K_{n,n}$ in --say-- the upper\footnote{Consider the graph with a bipartition into an upper and lower set of nodes like in Figure~\ref{fig:example1}.} node side, we select one edge among its incident ones uniformly at random. We remove this edge from $K_{n,n}$ and continue for the remaining nodes; this defines a random perfect matching. We repeat the above procedure $k$ times. If a node at the upper side becomes isolated before the completion of the above procedure, we repeat from scratch. Otherwise, the set of edges that have been removed constitutes the bundle graph. The second type of graphs consists of many components of small girth-$6$ graphs. For $k=p+1$, where $p$ is a prime, we use the $k$-regular bipartite graph with $k^2-k+1$ nodes per side whose construction is described in Section \ref{sec:prelim} and which was proved to be order-revealing in Lemma \ref{lem:ORGS-upper}. The bundle graph consists of multiple disconnected copies of this graph. Similarly, the third type of bundle graphs contains copies of the complete bipartite graph $K_{k,k}$ (possibly, containing one small non-complete $k$-regular bipartite graph if $k$ does not divide $n$). The selection of highly disconnected bundle graphs is intentional; these graphs are in a sense extreme (within their category) and can challenge our methods.

Table \ref{tab:perfect} depicts the data (percentage of correctly recovered pairwise relations) from the execution of Borda and RSD on $18$ distinct triplets of graph type and values\footnote{In all experiments reported here, $n$ equals or is very close to $1000$. This is because the results are essentially identical when significantly higher values of $n$ are used (up to $10,000$) and since the value of $1000$ has allowed us to complete our experiments in a reasonable time frame.} for the parameters $n$ and $k$. The data in the column labelled ``random $k$-regular'' show the average performance of Borda and RSD using $50$ random bundle graphs. A different random permutation is used each time in order to assign elements to nodes. For graphs of the second and third type, one graph is used for each pair of values for $n$ and $k$. For example, the data entries in the columns labeled ``girth-$6$'' and ``copies of $K_{k,k}$'' in the line with $k=3$ and $n=1001$ correspond to the performance of Borda and RSD on a girth-$6$ bundle graph which consists of $143$ copies of the $(7,3)$-bundle graph of Figure \ref{fig:example1}, and on a third-type graph that consists of $332$ copies of $K_{3,3}$ and one more $3$-regular graph with $5$ nodes per side. Again, the data are average performance values from $50$ executions; in each execution, a different random assignment of the elements to the nodes of the bundle graph is used. 

\begin{table*}[ht]
\centering
\begin{tabular}{||c c|c c|c c|c c||}
\hline
\multicolumn{2}{||c|}{graph} & \multicolumn{2}{|c|}{random $k$-regular} & \multicolumn{2}{|c|}{girth-$6$} & \multicolumn{2}{|c||}{copies of $K_{k,k}$} \\\hline
$k$ & $n$   & Borda & RSD & Borda & RSD & Borda & RSD \\
\hline\hline
$2$ & $1002$ & $73.3$ & $62.7$ & $73.5$ & $60.3$ & $66.8$ & $56.8$\\ 
\hline 
$3$ & $1001$ & $83.0$ & $77.2$ & $83.2$ & $66.0$ & $73.1$ & $60.2$ \\
\hline
$4$ & $1001$ & $87.5$ & $86.8$ & $87.7$ & $68.7$ & $77.1$ & $62.2$ \\ 
\hline 
$6$ & $1023$ & $92.0$ & $94.6$ & $92.1$ & $72.7$ & $81.6$ & $65.2$ \\ 
\hline 
$8$ & $1026$ & $94.2$ & $97.2$ & $94.1$ & $72.8$ & $84.3$ & $66.5$ \\ 
\hline 
$12$ & $1064$ & $96.3$ & $98.9$ & $96.6$ & $76.0$ & $87.3$ & $68.5$\\ 
\hline
\end{tabular} 
\caption{Performance of Borda and RSD with perfect grading on different bundle graphs of similar size.}
\label{tab:perfect}
\end{table*}

The results for Borda complement our theoretical analysis from Section \ref{sec:borda}. Indeed, the Borda-columns with bundle graphs of the second and third type indicate that the fraction of correctly recovered pairwise relations follows patterns of $1-\bigO(1/k)$ and $1-\bigO(1/\sqrt{k})$, respectively. Interestingly, the constants hidden in the $\bigO$ notation are significantly smaller than the theoretical constants $16\sqrt{3\pi}$ and $48\sqrt{2\pi}$, respectively. The results from the execution of Borda on random bundle graphs shows a pattern of $1-\bigO(1/k)$ as well, albeit with a slightly higher constant hidden in the $\bigO$ notation. We believe that this can be proved by extending our analysis in Section \ref{sec:borda}. Even though we have not managed to prove that the quantity $\eta(G)$ is $\bigO(k^2)$ for these graphs, we strongly believe that this is the case.

RSD has poor performance on bundle graphs of the second and third type. This can be easily explained by recalling that these bundle graphs consist of small connected components. Even though all pairwise relations between elements assigned to nodes of the same component are correctly recovered, the vast majority of the pairwise relations are between elements that are assigned to different components. The probability that such a relation will be recovered correctly is only $1/2$. This explains the small percentages in the second and third RSD-columns. 

In contrast, the first RSD-column (for random bundle graphs) shows a very interesting pattern. RSD is clearly worse than Borda for values of $k$ up to $4$ and becomes better as $k$ increases further. Actually, this is more apparent in Figure \ref{fig:pattern} where Borda and RSD are compared in $(n,k)$-bundle graphs for all values of $k$ from $2$ up to $25$ (and $n=1000$). Each data point in Figure \ref{fig:pattern} corresponds to the average performance among $50$ executions. Here, we can again recognize the $1-\bigO(1/k)$ pattern for Borda that was observed in Table \ref{tab:perfect} and we further conjecture an even better pattern of $1-\bigO(1/k^2)$ for RSD. Proving such a statement formally seems to be a challenging task. 
 
\begin{figure}[ht]
\centering
\includegraphics[trim=100 260 100 280, clip=true, scale=0.5]{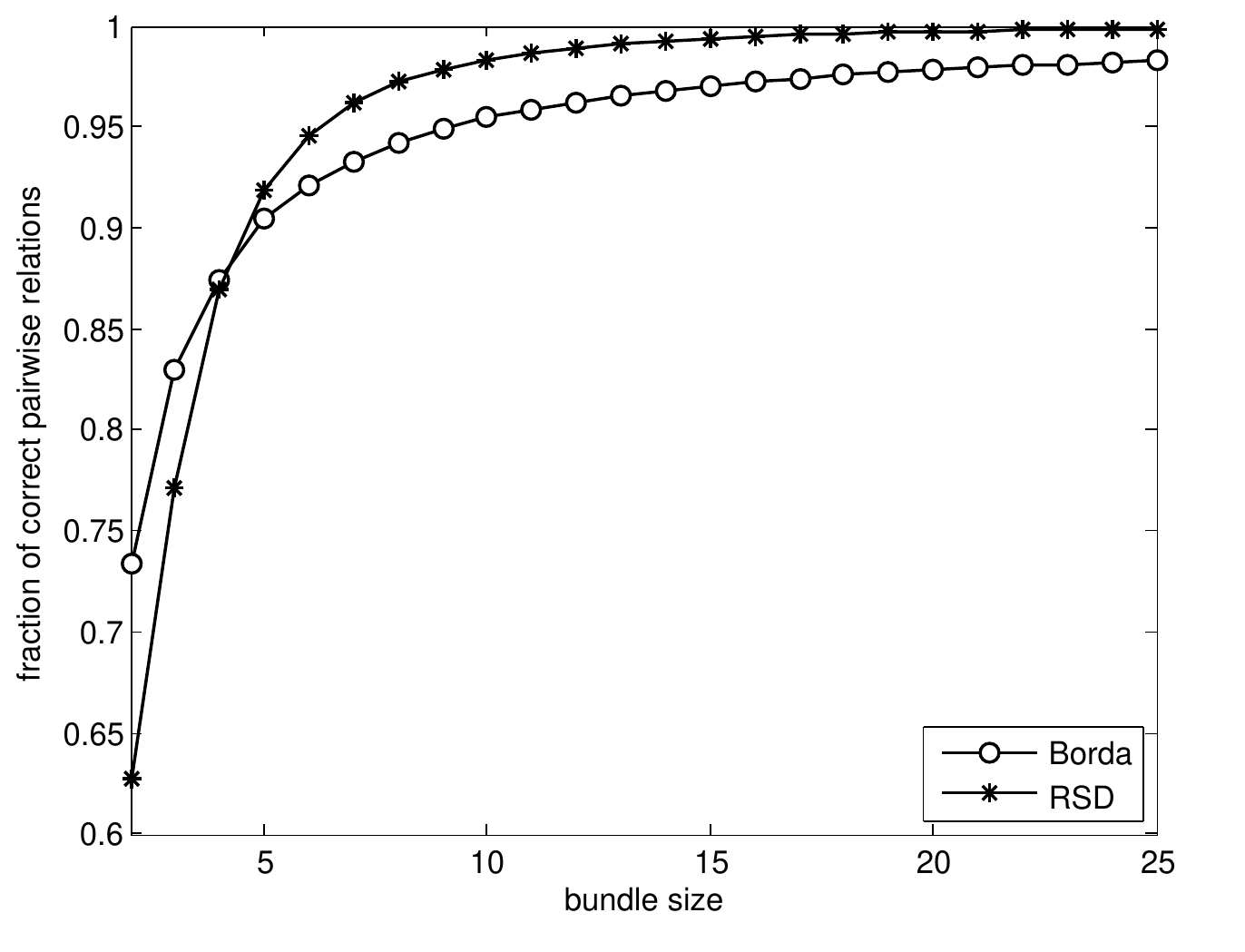}
\caption{Borda {\em vs.} RSD with perfect grading and bundle size ranging from $2$ to $25$.}
\label{fig:pattern}
\end{figure}

In a second set of experiments, we have studied imperfect grading. Now, we do not assume that the partial rankings are consistent to the ground truth any more. Instead, we have implemented generators of noisy rankings that may differ from the ground truth. In particular, we assume that each student has a quality that affects her position in the ground truth but also her ability to grade. First, the ground truth is the ranking of the elements in decreasing order of quality. Then, the ability of a student to rank the elements in a bundle depends on her quality $q$ and is modelled by the following process. For every pair of elements $a$ and $b$ in the bundle that is ranked as $a\succ b$ in the ground truth, decide the correct pairwise relation with probability $q$ and the opposite relation with probability $1-q$. If this process creates a circular pairwise relation, we repeat the whole process from scratch. Otherwise, the output induces a ranking in the obvious way; this ranking is the one computed by the student. Clearly, a student of quality $1$ will always produce a ranking that is consistent to the ground truth while a student of quality $1/2$ will produce a totally random ranking. This model was proposed by Condorcet in the 18th century; today, it is known as the Mallows model \cite{Mall57}.

In our experiments, we use different noise levels that indicate the range of student qualities. For example, a noise level of 30\% means that the qualities of the students are drawn uniformly at random from the interval $[0.7,1]$. We use random bundle graphs for different values of $k$ and besides Borda and RSD, we have also consider Markov chain-based aggregation methods. Dwork et al.~\cite{DKN+01} have studied a series of such methods; we describe the most powerful among them (even though we have experimented with a lot of variations of all the methods presented in \cite{DKN+01}), which is known as MC4. MC4 defines a Markov chain (or random walk) over the elements and ranks them in decreasing order of their probabilities in the stationary distribution of this chain. The transition matrix of the Markov chain is defined as follows: when at an element $a$, pick an element $b$ uniformly at random; if the number of partial rankings where $b$ is ranked above $a$ is higher than the number of partial rankings where $a$ is ranked above $b$, we have a transition to element $b$, otherwise we stay with element $a$.

Table \ref{tab:noisy} presents experimental data from the execution of Borda, RSD, and MC4 with random bundles for different values of the bundle size parameter and noise levels ranging from $50\%$ to perfect grading. RSD has poor performance for high noise levels and small values of $k$. For non-zero noise levels, Borda has the best performance. MC4 and RSD are good choices only in the case of perfect grading, with RSD outperforming MC4 for the high values of $k=8$ and $12$. Overall, our experiments suggest that Borda is extremely robust. Note that there are some values missing from Table \ref{tab:noisy}; this is due to the (exponential-time) implementation of Mallows generator which ``takes forever'' to come up with a set of non-circular pairwise relations that induces a ranking, when both $k$ and the noise level are high.

\begin{table*}[ht]
\centering
\begin{tabular}{||c|c c c|c c c|c c c||}
\hline
      & \multicolumn{3}{|c|}{$k=5$} & \multicolumn{3}{|c|}{$k=8$} & \multicolumn{3}{|c||}{$k=12$} \\\hline
noise level& Borda & RSD & MC4 & Borda & RSD & MC4 & Borda & RSD & MC4 \\
\hline\hline
$50$ & $81.6$ & $70.2$ & $78.4$ & $88.3$ & $74.0$ & $84.3$ & \#\#.\# & \#\#.\# & \#\#.\#\\ 
\hline 
$40$ & $84.9$ & $75.1$ & $81.2$ & $91.1$ & $80.1$ & $86.5$ & \#\#.\# & \#\#.\# & \#\#.\#\\
\hline
$30$ & $87.1$ & $80.0$ & $83.7$ & $92.6$ & $85.4$ & $88.3$ & \#\#.\# & \#\#.\# & \#\#.\#\\ 
\hline 
$20$ & $88.6$ & $84.2$ & $86.0$ & $93.5$ & $89.6$ & $89.8$ & $95.5$ & $92.2$ & $92.6$\\ 
\hline
$10$ & $89.6$ & $88.4$ & $88.8$ & $93.9$ & $93.2$ & $91.2$ & $96.1$ & $95.7$ & $93.6$\\ 
\hline 
$0$ & $90.4$ & $92.0$ & $92.7$ &$94.2$ & $97.2$ & $96.4$ & $96.2$ & $98.9$ & $97.8$\\ 
\hline
\end{tabular} 
\caption{Performance of Borda, RSD, and MC4 with random bundle graphs of size $1000$ and noise levels ranging from $50\%$ to perfect grading.}
\label{tab:noisy}
\end{table*}

We conclude by examining how sharply concentrated around the expectations the outcomes of the above experiments are. In Figure \ref{fig:spread}, we have plotted the fractions of correctly recovered pairwise relations obtained by Borda and RSD in the two extreme cases of perfect grading and noise level of $50\%$. Each figure contains data from $500$ executions (a random bundle graph and a random element-to-node assignment defines each execution) with $n=1000$ and $k=8$. The spread of fractions of correctly recovered pairwise relations achieved by Borda is almost the same in both cases. In contrast, RSD has a very high spread when the noise level is high (observe the long and narrow form of the left plot in Figure \ref{fig:spread}) while it is only marginally better than Borda in the perfect grading case. In conclusion, Borda appears to be robust with respect to this metric as well.

\begin{figure*}[ht]
\centering
\begin{subfigure}{0.43\textwidth}
	\caption{noise level of $50\%$}
    \includegraphics[trim=100 260 100 280, clip=true, width=\textwidth]{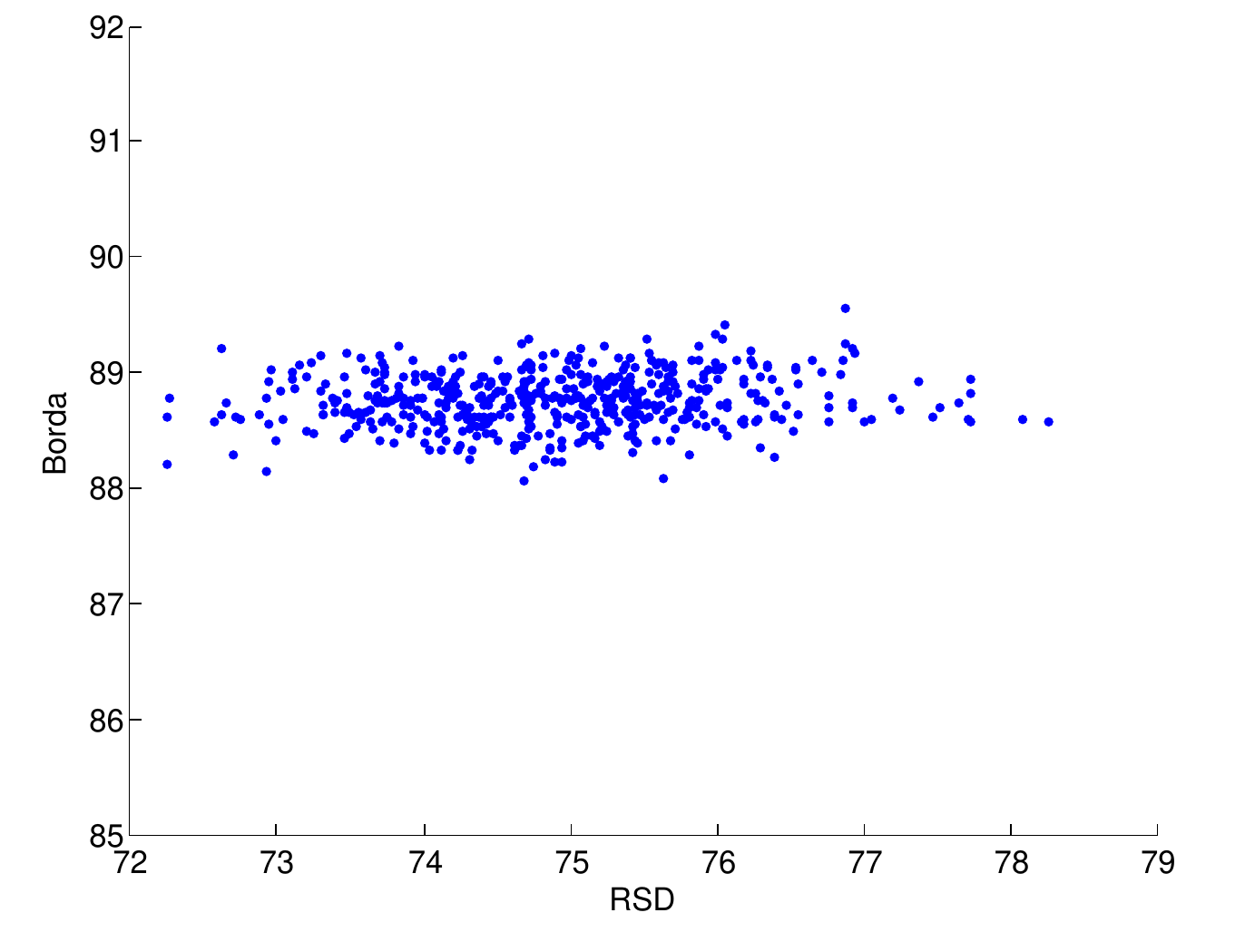}
\end{subfigure}%
\begin{subfigure}{0.43\textwidth}
	\caption{perfect grading}
    \includegraphics[trim=100 260 100 280, clip=true, width=\textwidth]{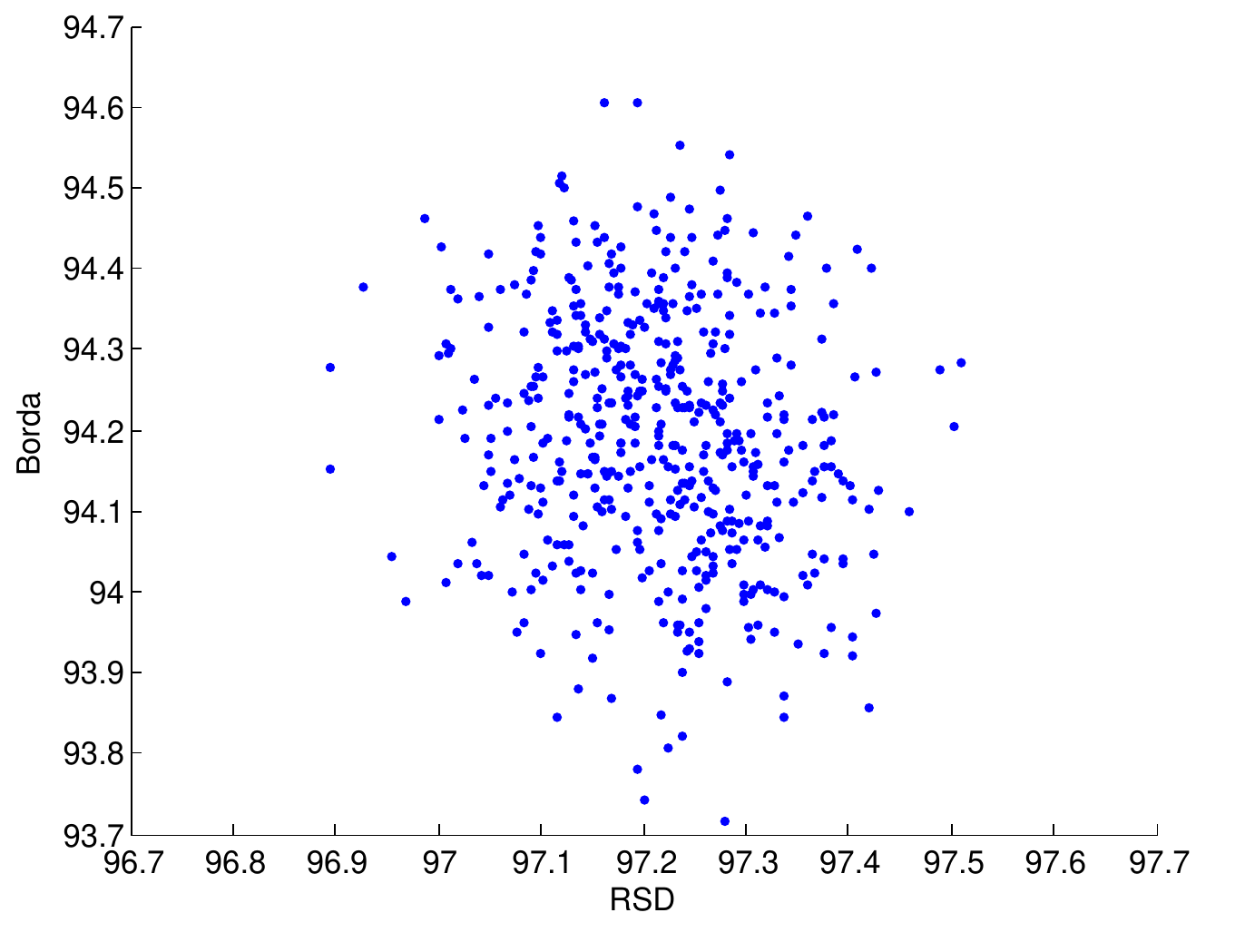}
\end{subfigure}
\caption{A comparison of Borda and RSD in $500$ executions for two different noise levels ($n=1000$, $k=8$).}
\label{fig:spread}
\end{figure*}

\section{Discussion}\label{sec:discussion}
Let us conclude by discussing some aspects of our work and possible future directions. Even though our analysis of Borda is targeted to the perfect grading case, we believe that our martingale-based arguments could be extended to handle imperfect grading under the Mallows noise model that we use in our experiments. This requires taking care of even more dependencies but we are confident that martingale theory will be useful here as well. We plan to consider extending our analysis in this direction in follow-up versions of this work.

Besides Borda, we have attempted a theoretical analysis of RSD as well. Here, our starting point has been to exploit the developments in the degree/diameter problem \cite{MS13} and use a diameter-$5$ low-degree bipartite graph as a bundle graph. The important property this graph has is that for every pair of nodes $u$ and $v$ of the node set $U$, these nodes either have a common neighbour in $V$ or there is another (intermediate) node $z$ in $U$ that has a common neighbour with $u$ and another common neighbour with $v$. Hence, in the perfect grading case, the pairwise relation between the elements $a$ and $b$ that are assigned to nodes $u$ and $v$ can be indirectly learnt during the serial phase through the pairwise relations of $a$ and $b$ with the element $c$ that is assigned to node $z$, provided that $c$ is ranked between $a$ and $b$ in the true ranking. Furthermore, if the bundle graph had more than one disjoint paths between any pair of nodes in $U$ (and more than one intermediates for any pair of nodes), the probability that the relation between two elements can be learnt correctly would be very high, provided that these elements have a relatively large rank difference in the true ranking. Unfortunately, even though some theoretical guarantees can indeed be formally proved in this way, the bundle graphs required have degree that strongly depends on the number of elements. So, this approach fails to explain the performance of RSD that we observed experimentally. Instead, one should reason about pairwise relations that can be learnt indirectly through long chains of intermediate elements. Unfortunately, exploiting such arguments seems elusive at this point.

In our experimental work, we have implemented and tested many more aggregation rules than the ones presented in Section \ref{sec:exp}. These include rules that put more weight on the partial rankings of low-rank (i.e., good) students. Such rules are usually defined using Markov chains that are variations of PageRank \cite{PBM+99} (such as the PeerRank method in \cite{Walsh14}), where the idea is that the confidence about the quality of a student depends on the performance of her graders (and this is reflected in the definition of the transition matrix of the Markov chain). Unfortunately, we have not observed any significant improvement compared to the rules considered in Section~\ref{sec:exp}. We believe that this can be explained by the fact that $k$ is a small constant.

In future work, we would also like to consider more realistic noise models that generalize  Mallows (see, e.g., \cite{FV86,LB11b,RJ14}) and ranking models that are inherently associated with cardinal utilities such as the generalized random utility model of Soufiani et al. \cite{SPX13} (see also the book of \cite{Liu11} and the references therein). Of course, it is important to perform real-world experiments (with students in the classroom or with participants in real MOOCs, if possible) in order to justify our methods and determine the noise model that is closest to practice.

\paragraph{Acknowledgements.} We would like to thank Stratis Gallopoulos and Steve Vavasis for discussions on early stages of this work, and Panagiotis Kanellopoulos and Nisarg Shah for technical comments and remarks.

\bibliography{moocs.arxiv}  

\end{document}